\documentclass[format=acmsmall, review=false]{acmart}
\usepackage{acm-ec-22}
\usepackage{booktabs} 
\usepackage[ruled]{algorithm2e} 

\SetAlFnt{\small}
\SetAlCapFnt{\small}
\SetAlCapNameFnt{\small}
\SetAlCapHSkip{0pt}
\IncMargin{-\parindent}
\DeclareMathOperator{\Exp}{\mathbb{E}}
\newcommand\ind[1]{1_{\left[#1 \right]}}
\setcitestyle{authoryear}

\title[Reciprocity in Machine Learning]{Reciprocity in Machine Learning}

\author{Mukund Sundararajan (Google), Walid Krichene(Google Research)}
\date{Feb 2022}

\usepackage{natbib}
\usepackage{graphicx}
\usepackage{amsmath}
\usepackage{amsthm}
\usepackage{subfig}

\newtheorem{proposition}{Theorem}[section]
\newtheorem{remark}{Remark}[section]

\usepackage{todonotes} 
\newcommand\trackin{\texttt{TracIn}}
\newcommand\marginal{\texttt{Marginal}}
\newcommand\influence{\texttt{Influence}}

\begin{abstract}
  
  Machine learning is pervasive. It powers recommender systems such as Spotify, Instagram and YouTube, and health-care systems via models that predict sleep patterns, or the risk of disease. Individuals contribute data to these models and benefit from them. Are these contributions (outflows of influence) and benefits (inflows of influence) reciprocal?
  
We propose measures of outflows, inflows and reciprocity building on previously proposed measures of training data influence. Our initial theoretical and empirical results indicate that under certain distributional assumptions, some classes of models are approximately reciprocal.

We conclude with several open directions.
\end{abstract}

\begin{document}
\maketitle

\section{Introduction}

\subsection{Machine Learning and User Data}

Machine Learning uses training data to build models that produce useful predictions on unseen data points. Sometimes training data comes from a population of individuals, and the model is used to benefit the same population. Consider recommender systems such as Spotify, Instagram or YouTube that recommend items (e.g. music, videos, posts) to their users. Past interactions of users with the items are used to learn preferences of users and characteristics of content. Alternatively, consider initiatives that collect health data from several individuals and use this to build models that predict sleep, or the risk of disease~\cite{fib,dr,cancer, hospitalization, death}. 

\subsection{Reciprocity}
\label{sec:reciprocity}

In this paper we are \emph{only} focused on value exchange that happens via the machine learning. For recommendation systems, we are not concerned with the value that arises from the creation of content, just that from the implicit curation that happens via the machine learning. 

There is a flow of influence from training data points to predictions. For a specific individual (the protagonist), we can partition these flows three ways. 

There is the \emph{self-influence}, i.e., the benefit of the protagonist's data on their own predictions. For instance in a recommender system, an individual's past interactions with items (content) can be used to learn their preferences, which is used to serve them recommendations.

Then, there is the \emph{outflow}, i.e., the influence of the protagonist's training data on other individuals. For instance in a recommender system, an individual's past interactions with content can be used (in conjunction with data from other individuals) to learn characteristics of items, which are then used to serve recommendations to other individuals. 

Finally, there is the \emph{inflow}, which is just the converse of the outflow, i.e., the influence of other individual's data on the protagonist.  

The goal of this paper is to discuss \emph{the balance between outflows and inflows}, i.e., whether they are \emph{reciprocal}. We will not be too concerned with self-influence.

Reciprocity does not require that the outflows (or inflows) across individuals are equal or similar. Indeed, in a recommender system a heavy user is likely to have larger outflows (and inflows) in comparison to a light user. Reciprocity only requires that the outflows and inflows are balanced individual-by-individual, i.e., an individual who contributes a lot, benefits a lot, and a individual who contributes a little also benefits a little.\footnote{Contrast this with differential privacy~\cite{dpsgd}, where one seeks to control the magnitude of outflows.}

Reciprocity does not even require that inflows and outflows be positive. Indeed, in a recommender system, the data of an individual with atypical tastes may hurt the recommendations to other individuals. Reciprocity only demands that if inflow is negative, outflow should be equally negative.

\subsubsection{Why study reciprocity?}

Reciprocity has been studied in various fields. Sociologists have identified reciprocity as a social norm that pervades societies over the ages and contexts. Gouldner~\cite{norm} quotes the Roman statesman/philosopher Cicero  as saying: ``There is no duty more indispensable than returning an act of kindness". It has been discussed as a basis of law by Rawls~\cite{rawls} who says that citizens must believe that other citizens can reasonably accept the enforcement of a particular set of basic laws for them to follow the law. In economics, Fehr and Gächter~\cite{Fehr} point out that reciprocity is a complement to contracts---it is necessary for enforcement. They also point out that reciprocity can be a substitute for contracts, that contracts remain incomplete or are absent when there is already sufficient reciprocity. Reciprocity is also a concern in matching markets such as kidney exchanges~\cite{kidney} and student exchange programs~\cite{student-exchange,Erasmus}; Reciprocity is essential to incentivize these markets, and is accomplished by protocol design or by the payment of monies.   

Our motivation is mostly related to the sociological one. As we discussed earlier, machine-learning models pervade our online existence. We both contribute data, and we benefit. So is there balance between these?

\subsection{Sources of Non-reciprocity}

Some classes of machine learning models may be inherently non-reciprocal. For example, consider k-Nearest Neighbors (kNN).

\begin{remark} [k-Nearest Neighbors]
 In kNN, there is a distance function $d$ on points in feature space and the prediction of the algorithm is the mean (for regression) or majority (for classification) of the $k$ nearest neighbors as per the distance function $d$ of the prediction point. 
 
Suppose there is an outlier individual with outlier data points in the feature space. This individual will not have any influence on the other individuals because their training examples will not appear in the top $k$ nearest list for prediction points of other individuals. However, their predictions will benefit from the training data of other individuals.
\end{remark}

The time sequence of individual interactions can also affect the strength of reciprocity.

\begin{remark}
\label{rem:arrow}
Machine learning operates on training data. Training data comes from past interactions of the individuals with the system (referred to as the training set). The model is then used to benefit future interactions (referred to as the deployment set). This process defines the flow of contributions from the past to the future. If all of an individual's interactions with the system are late in the system's lifetime, this individual will benefit but not contribute.
\end{remark}

To control for time effects, we make the following stationarity assumption for our theorems and experiments:

\begin{remark}
\label{rem:dist}
There is a population $U$ of individuals and a set $X$ of features and a set $Y$ of labels. We assume that the training and deployment sets are both drawn IID from a joint distribution over $U \times X \times Y$.

Notice that this allows for individuals to differ in characteristics and in the frequency with which they interact with the system. The assumption requires that an individual's data be equally likely to occur in training or deployment.

In practice, this stationarity may not hold. The data distribution can change: in the recommendation example, a new user might contribute few examples to the training set but have many more in the deployment set.
\end{remark}

\subsection{Contributions of Work}

First, we propose a measure of reciprocity using techniques that quantify the influence of training data points on predictions~\cite{tracin, hat}. Such a technique produces a measure of the influence of a single training example on a single prediction. We use this to measure the inflow (benefit) and outflow (contributions), and then define a measure of reciprocity. A model is $\alpha$-reciprocal for an individual if the ratio of inflowing benefit to outflowing contributions for the individual are in the range $[\alpha, 1/\alpha]$ for some $\alpha \in [0,1]$.

Then, we prove a theorem that shows that under certain distributional assumptions, models trained using stochastic gradient descent are strongly reciprocal. The theorem indicates that when the training and deployment sets are drawn from the same joint distribution, we should expect reciprocity for such models.

We perform experiments on one recommendation and two healthcare data sets. We find that reciprocity is relatively high across all experiments in expectation. In one of the experiments (in which each user contributes many data points), we find this to be true even for one realization.

\section{Preliminaries}

\subsection{Recommender Systems}
\label{sec:recsys}
One class of machine learning models that we study power recommender systems. We use this as a running example. 

We investigate a standard model~\cite{survey} of recommender systems. There is a population $U$ of individuals (indexed by $u$) and a set $I$ of items (pieces of content, indexed by $i$). When individuals interact with items, they \emph{rate} the item implicitly or explicitly. This produces a score $r_{ui}$. The individuals and the items may have features associated with them. These could be ids, or descriptive features such as the genre of the movie, the location of the individual etc. The task is to predict the scores for unseen interactions given scores for past interactions.

We seek to measure the value exchange from the implicit \emph{curation} that occurs when individuals consume recommendations. We do not measure the value arising from the \emph{creation} of content. 

\subsection{Machine Learning}
\label{sec:ml}

\subsubsection{General Setup}

There is a training data set $Z$ of labelled examples. The model is learned using this data and applied to a deployment set $Z'$. In our recommender system example, $Z$ consists of past interactions between individuals and items, and $Z'$ consists of (future) predictions. In practice, the machine learning model may be rebuilt over new training data sets; our model of machine learning represents a snapshot in time.

Every example $z=(x,y)$ consists of features $x$ that the machine learning model uses to learn patterns, and a label $y$; this is the prediction target or the response variable. 

For a recommender system, the features are characteristics of individuals and items, and the label is a rating $r_{ui}$ for the individual-item pair $u,i$ associated with the example.

\subsubsection{Model Architectures and Training Algorithms}


We study models trained using Gradient Descent. The model is parameterized by a weight vector $w \in \mathbb{R}^p$. This could be a simple linear/logistic model, or a deep-learning model of arbitrary architecture. The training process minimizes a loss function: $\ell: \mathbb{R}^p \times Z \rightarrow \mathbb{R}$; thus, the loss of a model on an example $z$ is given by $\ell(w, z)$. One commonly used loss function in the recommendation literature is squared error $(r-\hat{r})^2$. In (Stochastic) Gradient Descent, the weight parameters are updated iteratively, as the training process visits randomly selected batches of training examples. Suppose a batch $B_t$ is visited at step $t$, the parameters are updated in the direction opposite (we are minimizing loss) to the gradient: $\sum_{z \in B_t}  \eta_t \nabla \ell(w_t, z)$; here $\nabla \ell(\cdot, z)$ is the gradient of the loss function with respect to the weight parameters, $w_t$ is the value of the weight parameters at time $t$, and $\eta_t$ is the step-size at time $t$. 

If the data set is small, every batch can visit the entire data set $Z$, i.e. $B_t = Z$; this is called \emph{full-batch gradient descent}. \footnote{However, in large data sets, it is often computationally advantageous to process a subset of training examples at each step. Therefore the training process operates on randomly chosen, i.e., \emph{stochastic} batches of examples.}

\section{Measuring Reciprocity}
\label{sec:measuring}

We define measures of reciprocity. Our measure is parameterized by a technique that identifies the \emph{influence} of a training example on a  prediction example. Formally, an influence computation technique measures the impact of a specific training example $z$ (belonging to one individual)  on a prediction made by the model on a deployment example $z'$ (belonging to another individual). There are several ways to measure influence, and this affects the results; in this paper, we investigate a couple of methods.

\subsection{Marginal Influence}

Let $M_Z$ be the machine learning model trained on the data set $Z$; $M_Z(x)$ is the model's prediction on a data point $x$. The marginal influence of a training point $z$ on a prediction point $z'=(x', y')$ is based on the \emph{counterfactual} of removing the point $z$ from the training set:

\begin{align}
 \marginal(z, z') =  \ell(M_{Z \setminus \{z\}}(x'),y') - \ell(M_Z(x'),y')  
  \label{eq:marginal}
 \end{align}
 
We define all our influence measures so that they measure \emph{reduction in the loss} of the deployment example from the presence of a training example. Thus a positive quantity connotes that loss was reduced by that amount, and a negative quantity connotes that loss was increased by that magnitude.

Computing exact marginal influence often requires retraining the model on the modified data set. However, this can be approximated, without retraining, via certain Hessian approximations~\cite{inluence}.

\begin{remark} [Signal-to-noise ratio in Marginal]
\label{rem:marginal-noise}
\marginal{} relies on deleting a single individual's data and retraining the model to optimality. This may result in a substantially similar model, if the data set is large. This makes this measure of influence more susceptible to noise. This is confirmed in our experiments.
\end{remark}

\subsection{\trackin{} Influence}
\label{sec:influence}

The second influence measure we  use is called \trackin{}~\cite{tracin}. Whereas \marginal{} relies on a counterfactual approach, \trackin{} assigns contributions and benefits based on \emph{actual} work done during the training process. It is therefore reliant on the training process, and applies only to models trained using Stochastic Gradient Descent.

Here is the idea: First, suppose that SGD visits examples one at a time, i.e., the batch size ($B_t$) is one. Then, the visit (to the training example $z$) changes the model parameters, and this changes the model's loss on the prediction example $z'$. It is natural to attribute this change in loss to the contribution of example $z$ and the benefit of example $z'$

However, the model visits several training examples at once. Therefore, we have to disentangle the outflows of the examples $z \in B_t$. \trackin{} does this using dot products of gradients: $-\eta_t \nabla \ell(w_t, z) \cdot  \nabla \ell(w_t, z')$. The term $-\eta_t \nabla \ell(w_t, z)$ captures the movement in the weight parameters due to example $z$; this is by definition of gradient descent. And $\nabla \ell(w_t, z')$ models change in the loss of the prediction example $z'$ due to a change in the weight parameters. The influence of the training example $z$ on a prediction example $z'$ is computed by summing across all the batches in which the example is present:  

\begin{align}
 \trackin(z, z') = \sum_{t:\ z \in B_t}  \eta_t \nabla \ell(w_t, z') \cdot  \nabla \ell(w_t, z).
  \label{eq:tracin}
 \end{align}


\begin{remark} 
\label{rem:first-order}
The use of gradients entails a first-order approximation. The actual change in loss of the example $z'$, can be written as $\ell(w_{t+1}, z') = \ell(w_t, z') + \nabla \ell(w_t, z') \cdot (w_{t+1} - w_t)$, plus a higher order term of order $O(\eta_t^2)$, that we ignore. This approximation is reasonable when the step-sizes are small. We measure this discrepancy for our experiments (see Figure~\ref{fig:approximation}).
\end{remark}

\subsection{A Measure of Reciprocity}
\label{sec:measure}

We now define a measure of reciprocity. Let $Z_u$ and $Z'_u$ be the set of training and deployment examples belonging to individual $u$. Then, the influence of other individuals' ($v \neq u$) data on the predictions for individual $u$, i.e., the \emph{inflow} is: 
\begin{align}
\label{eq:inflow}
I_u = \sum_{z \in Z \setminus Z_u} \sum_{z' \in Z'_u} \influence(z, z').   
\end{align}

Here Influence can either be \trackin{} influence or \marginal{} influence. Similarly the \emph{outflow} (contribution) of individual $u$ towards other individuals' predictions is:
\begin{align}
\label{eq:outflow}
O_u = - \sum_{z' \in Z \setminus Z'_v} \sum_{z \in Z_u} \influence{}(z, z').  
\end{align}

As discussed in the introduction, inflow and outflow can be positive (an individual benefits or is benefited from participating in the system) or negative (an individual is harmed or harms others by participating in the system).

We say that a machine learning model is \textbf{$\alpha$-reciprocal} for a individual $u$ if the ratio of outflow to inflow $I_u/O_u$ is in the range $[\alpha, 1/\alpha]$ for $\alpha$ in the range $[0,1]$. If the signs of $I_u$ and $O_u$ do not match, we define reciprocity to be $0$.

We say that a  model is $(p,\alpha)$-reciprocal if it is $\alpha$-reciprocal for $p$ fraction of individuals. (Thus, a model is at best $(1,1)$ reciprocal and at worse $(0,0)$ reciprocal.) Depending on the measure of influence that we use, we will say $(p, \alpha)$-\trackin-reciprocal or $(p, \alpha)$-\marginal-reciprocal. Another measure we study is the correlation between the inflow and the outflow across individuals; this measure emphasizes individuals with large inflows and outflows.


\subsubsection{Computing \trackin{} Inflows and Outflows}
\label{rem:computation} 
The definition of Inflows and Outflows as a sum of terms $\trackin(z, z')$ may suggest that one needs to compute this matrix of pairwise \trackin{} influence, which would incur a $O(|Z||Z'|p)$ computational cost, where $p$ is the number of parameters of the model. This is not the case: since \trackin{} is a sum of dot products, one can rewrite the Outflow as follows.
\begin{align}
O_u 
&= -\sum_{z \in Z_u} \sum_{z' \in Z' \setminus Z'_u} \trackin(z, z') \notag\\
&= -\sum_{z \in Z_u} \sum_{z' \in Z' \setminus Z'_u} \sum_{t: z_t = z} \eta_t \nabla \ell(w_t, z')\cdot \nabla \ell(w_t, z) \notag\\
&= -\sum_{z \in Z_u} \sum_{t: z_t = z} \eta_t \nabla \ell(w_t, z) \cdot \sum_{z' \in Z' \setminus Z'_u} \nabla \ell(w_t, z') \notag\\
&= -\sum_{z \in Z_u} \sum_{t: z_t = z} \eta_t \nabla \ell(w_t, z) \cdot (\nabla \ell(w_t, Z') - \nabla \ell(w_t, Z'_u)), \label{eq:outflow_efficient}
\end{align}
where, in the last equality, we define
\begin{equation}
\label{eq:sum_gradients}
\nabla \ell(w_t, Z') = \sum_{z' \in Z'} \nabla \ell(w_t, z').
\end{equation}
Notice that this term does not depend on $u$. It can be computed once in $O(|Z'|p)$, and reused for computing $O_u$ for all $u$. The total complexity is thus reduced to $O(|Z|p + |Z'|p)$. A similar observation holds for inflows.


\subsection{Inflow and Outflow for Matrix Factorization}
\label{subsec:influence_mf}

We discuss the concepts of influence, outflow, inflow and reciprocity for a specific type of machine learning model based on \emph{matrix factorization}, in the context of the recommendation running example.

In such a model, every individual $u$ is endowed with a $d$ dimensional vector $p_u$ and every item $i$ is endowed with a $d$ dimensional vector $q_i$; we refer to these vectors as embeddings. The prediction matrix is the product $\hat R = PQ^\top$, where $p_u$ is the $u$-th row of matrix $P$ and similarly for $Q$. In other words, the predicted rating for user-item pair $(u, i)$ is given by the dot product $\hat r_{u,i} = p_u \cdot q_i$. See~\cite{matrix-factorization} for more details about matrix factorization.

The model is optimized for a regularized quadratic loss, i.e.,
\[
\frac{1}{2}\sum_{(u, i) \in Z} (p_u \cdot q_i - r_{ui})^2 + \frac{\lambda}{2} \left( \sum_u \|p_u\|^2 + \sum_i \|q_i\|^2 \right),
\]
where $r_{ui}$ is the label of pair $(u,i)$. The regularization term helps generalization; intuitively, it ensures that the model does not overfit on rare  users or rare items.

For the purpose of attribution, we decompose the regularization term as a sum over training examples, and define the loss as
\[
\frac{1}{2}\sum_{(u, i) \in Z} \left((p_u \cdot q_i - r_{ui})^2 + \frac{\lambda}{|Z_u|} \|p_u\|^2 + \frac{\lambda}{|Z_i|} \|q_i\|^2 \right),
\]
where $Z_u = \{i : (u,i) \in Z\}$ and $Z_i = \{u: (u, i) \in Z\}$.

Notice that the loss gradients and the \trackin{} influence have a simple structure. A visit to a training example $(u, i)$ only updates the vectors $p_u$ and $q_i$. Moreover, the change in the user embedding $p_u$ only affects the predictions for user $u$; therefore updates to the user vectors do not play a role in the definitions of inflow and outflow. (Equations~\ref{eq:inflow} and~\ref{eq:outflow}).

The update to the item embedding $q_i$ only influences users who interact with the item $i$ in the deployment set. Thus, for this loss function, inflows and outflows \emph{only flow through updates to item embeddings}.

\begin{remark}[Computing \trackin{} for Matrix factorization]
Section~\ref{rem:computation} suggests that we can compute \trackin{} in $O(|Z|p + |Z'|p)$ where $p$ is the total number of model parameters, in this case $p = d(|U| + |I|)$. But due to the structure of the problem, this computation can be done more efficiently for matrix factorization. Notice that the gradient of the loss w.r.t. a training example $z = (u, i)$ is $2d$-sparse, only the embeddings $p_u, q_i$ have a non-zero gradient. Thus, computing $\nabla \ell(w_t, Z)$ in Equation~\eqref{eq:sum_gradients}, can be done in $O(|Z|d)$ instead of $O(|Z|p)$. Similarly, computing the sum of dot products in Equation~\eqref{eq:outflow_efficient} requires $O(|Z'|d)$ operations (since each $\nabla \ell(w_t, z)$ is $2d$-sparse). The total complexity is therefore $O(|Z|d + |Z'|d)$. This is equal to the complexity of running gradient descent, which means that computing Inflows and Outflows along the SGD trajectory does not significantly increase the computational cost of model training.
\end{remark}


\section{Theoretical Results}
\label{sec:theory}

\begin{proposition}
\label{thm:main}
Consider a model trained with Stochastic Gradient Descent for $T$ steps. Suppose that the training and deployment sets satisfy the assumption of Remark~\ref{rem:dist}, i.e. that they are both drawn IID from the same distribution. Furthermore, suppose that batches of training data $(B_t)_{t \in \{1, \dots, T\}}$ are mutually independent. Then the model is $(1,1)$-\trackin-reciprocal in expectation, in the sense that for all individuals $u$, $\Exp[I_u] = \Exp[O_u]$.
\end{proposition}

The proof will formalize the following intuitive argument: given a pair of users $(u, v)$, the \trackin{} influence at time $t$ of $z \in Z_u$ on $z' \in Z'_v$ is $\eta_t\nabla\ell(w_t, z')\cdot\nabla\ell(w_t, z)$, which is symmetric in $z, z'$. Conditioned on the model parameters at time step $t$, training and deployment examples have the same distribution, so the expected influence of $u$ on $v$ at time $t$ is the same as the expected influence of $v$ on $u$. Summing over $t$ and over users $v \neq u$ concludes the argument.
\begin{proof}
\def\zcal{\mathcal{Z}}
First, we introduce some notation. Let $\zcal = U \times X \times Y$, where $U$ is the population of individuals, $X$ is the feature set and $Y$ is the label set. Let $D$ be the joint distribution over $\zcal$. For an individual $u \in U$, we write $\zcal_u = \{u\} \times X \times Y$, so that a training example $z$ belongs to individual $u$ if $z \in \zcal_u$.

Now, given batches of training data $B_1, \dots, B_T$ and a deployment set $Z'$ (note that both are random variables), we rewrite inflows and outflows in a form that is more amenable to taking expectations.
\begin{align}
I_u
&= \sum_{t = 1}^T \eta_t \sum_{z \in B_t : z \notin \zcal_u} \sum_{z' \in Z'_u} \nabla\ell(w_t, z) \cdot \nabla \ell(w_t, z') \notag\\
&= \sum_{t = 1}^T \eta_t \left(\sum_{z \in B_t} \nabla\ell(w_t, z) \ind{z \notin \zcal_u}\right) \cdot \left(\sum_{z' \in Z'} \nabla \ell(w_t, z') \ind{z' \in \zcal_u}\right), \label{eq:proof_in}
\end{align}
where $\ind{z' \in \zcal_u}$ is the indicator of the event ``$z'$ belongs to user $u$''. Similarly, we have for outflows
\begin{equation}
O_u = \sum_{t = 1}^T \eta_t \left(\sum_{z \in B_t} \nabla \ell(w_t, z_t) \ind{z_t \in \zcal_u}\right) \cdot \left(\sum_{z' \in Z'} \nabla \ell(w_t, z') \ind{z' \notin \zcal_u} \right).
\label{eq:proof_out}
\end{equation}

Let $(F_t)$ denote the filtration arising from the sequence of random variables $B_1, \dots, B_t$. Taking the expectation of inflow in Equation~\eqref{eq:proof_in}, and using the tower property of conditional expectations, we have
\[
\Exp[I_u] = \Exp\left[ \sum_{t = 1}^T \eta_t \Exp\left[\sum_{z \in B_t} \nabla\ell(w_t, z) \ind{z \notin \zcal_u} \cdot \sum_{z' \in Z'} \nabla \ell(w_t, z') \ind{z' \in \zcal_u} \Big| F_{t-1} \right]\right].
\]
Now, notice that the batches $(B_1, \dots, B_{t-1})$ completely determine the model parameters $w_t$ (since $w_t = w_0 - \sum_{\tau = 0}^{t-1} \eta_\tau \sum_{z\in B_\tau}\nabla \ell(w_\tau, z)$), and by assumption, the next batch of training examples $B_t$ is independent of previous batches, and so is the deployment set $Z'$. So conditioned on $F_{t-1}$, the two random variables $\sum_{z \in B_t}\nabla\ell(w_t, z)\ind{z \notin \zcal_u}$ and $\sum_{z' \in Z'}\nabla \ell(w_t, z')\ind{z' \in \zcal_u}$ are independent. Let us denote by
\[
g_{t-1}^{-u} = \Exp_{z \sim D}[\nabla \ell(w_t, z) \ind{z \notin \zcal_u}|F_{t-1}], \quad g_{t-1}^{u} = \Exp_{z' \sim D}[\nabla \ell(w_t, z') \ind{z' \in \zcal_u}|F_{t-1}].
\]
Then, by the aforementioned independence, linearity of expectations, and the assumption that elements of $Z'$ and $B_t$ follow the same distribution $D$, we have
\[
\Exp[I_u] = \Exp\left[ \sum_{t = 1}^T \eta_t (|B_t|g^{-u}_{t-1}) \cdot (|Z'|g^u_{t-1})\right].
\]
We make a similar calculation for outflows (the only difference is in the indicators): taking expectations in Equation~\eqref{eq:proof_out},
\begin{align*}
\Exp[O_u] &= \Exp\left[ \sum_{t = 1}^T \eta_t \Exp\left[ \sum_{z \in B_t} \nabla \ell(w_t, z) \ind{z \in \zcal_u} \cdot \sum_{z' \in Z'} \nabla \ell(w_t, z') \ind{z' \notin \zcal_u} \Big| F_{t-1} \right]\right],\\
&= \Exp\left[ \sum_{t = 1}^T \eta_t (|B_t|g_{t-1}^u) \cdot(|Z'|g_{t-1}^{-u})\right],
\end{align*}
where we used independence (conditional on $F_{t-1}$) of the random variables $\sum_{z \in B_t} \nabla \ell(w_t, z) \ind{z \in \zcal_u}$ and $\sum_{z' \in Z'} \nabla \ell(w_t, z') \ind{z' \notin \zcal_u}$. The last expression is equal to $\Exp[I_u]$. This concludes the proof.
\end{proof}

\begin{remark}
The theorem assumes that at each step of gradient descent, a new set of independent examples is drawn, which precludes revisiting the same example multiple times. In practice, training examples are revisited, which breaks independence. But notice that for the result to hold, the proof only needs that future samples $B_t$ be independent of the past trajectory $w_0, \dots, w_t$. In some regimes, this may be a reasonable approximation. For example, when the batch size is very large, there is little variance in the gradients, and one can informally treat the trajectory $w_0, \dots, w_t$ as being deterministic, unaffected by random sampling of training data. Another regime is when the data set is very large, and training only requires a very small number of passes over the training data. In such cases, independence may be a reasonable approximation. Our experiments suggest that reciprocity may be preserved to a large extent even when the independence assumption is broken.
\end{remark}

The theorem suggests that reciprocity holds whether the inflows and outflows are positive or negative; indeed an individual's data could hurt another's predictions if their characteristics are very different. (In our model, negative outflows and inflows would manifest as negative dot-products $\nabla \ell(w_t, z_u) \cdot \nabla \ell(w_t, z_v)$, what helps one individual hurts the other.) It holds irrespective of the variation in inflows and outflows across individuals; indeed some individuals may contribute more data (the set $Z_u$ is large) and others, very little (the set $Z_u$ is small).

\begin{remark}
\label{rem:clipping}
Reciprocity may not hold if the training process normalizes or clips (rescales) gradients. This can break reciprocity because the modification affects the gradient update, but it does not affect the gradient of the deployment loss. This breaks the symmetry in Equation~\eqref{eq:tracin}. For example, clipping is used to enforce differential privacy~\cite{dpsgd}. Its purpose is to control the contribution of any individual to the model, it is therefore expected that reciprocity is broken. Clipping and normalization are also used for stabilizing training of deep neural networks~\cite{clipping}. This may also affect reciprocity.
\end{remark}

\section{Experiments}
\label{sec:empirical}

We now perform experiments on a recommender system data set and two health data sets. For each experiment, we randomly partition the data set into training and deployment sets. Measurements are averaged across several such random splits. Models are trained using many iterations of gradient descent. This experimental setup relaxes the assumption of Theorem~\ref{thm:main}. The random partitioning mimics the assumption that the train and deployment sets are identically distributed. However, the independence assumption does not hold since training examples are revisited multiple times.

We find that \trackin{}-reciprocity is relatively high across all experiments in expectation, i.e. when averaged across splits. For MovieLens, reciprocity is high even for a single split.

\subsection{Recommender Systems: MovieLens}
We conduct experiments on MovieLens Data~\cite{MovieLens}, specifically, the MovieLens 100K data set with 943 individuals, 1682 items (movies), and 100,000 ratings, i.e., an average of about 106 ratings per individual.  Each individual has at least 20 ratings. Each movie is rated on a scale from 0-5. We randomly split the ratings into training and deployment sets in the ratio 80:20.


We train a matrix factorization model using Gradient Descent. We randomly initialize the user and item embeddings. For our experiments, we use a fixed embedding dimension $d=16$. 

Given the relatively small size of the training data, we will use full-batch Gradient Descent, i.e., every example is visited at every time step.

We use the following hyper-parameters, which we tuned on a random split of the data: a regularization coefficient $\lambda = 1$, a number of steps $T = 1000$, and a learning rate $\eta = 0.0002$. With these parameters, we train the model on ten different random splits, each repeated ten times (to average across random initializations). The average test root mean squared error (RMSE) is $0.925$ with a standard deviation of $0.009$.\footnote{Matrix factorization on MovieLens 100K is used as baseline in many works on recommender systems. As a sanity check, we compare to such works. The best reported RMSE we could find for matrix factorization is 0.911~\cite{autosvd++,attribute} using a 90-10 split. Our model (with the same hyper-parameters), has an RMSE of 0.910 on the 90-10 split.}

\subsubsection{Marginal Influence is Noisy}

We now study reciprocity based on \marginal{} influence on the MovieLens data set. We find that the resulting measures of outflow and inflow have a low signal-to-noise ratio (SNR), defined as the mean divided by the standard deviation, both computed across different random splits. (Recall Remark~\ref{rem:marginal-noise}).

\begin{figure}
    \centering
    \subfloat[\centering \marginal{}]{\includegraphics[width=0.45\textwidth]{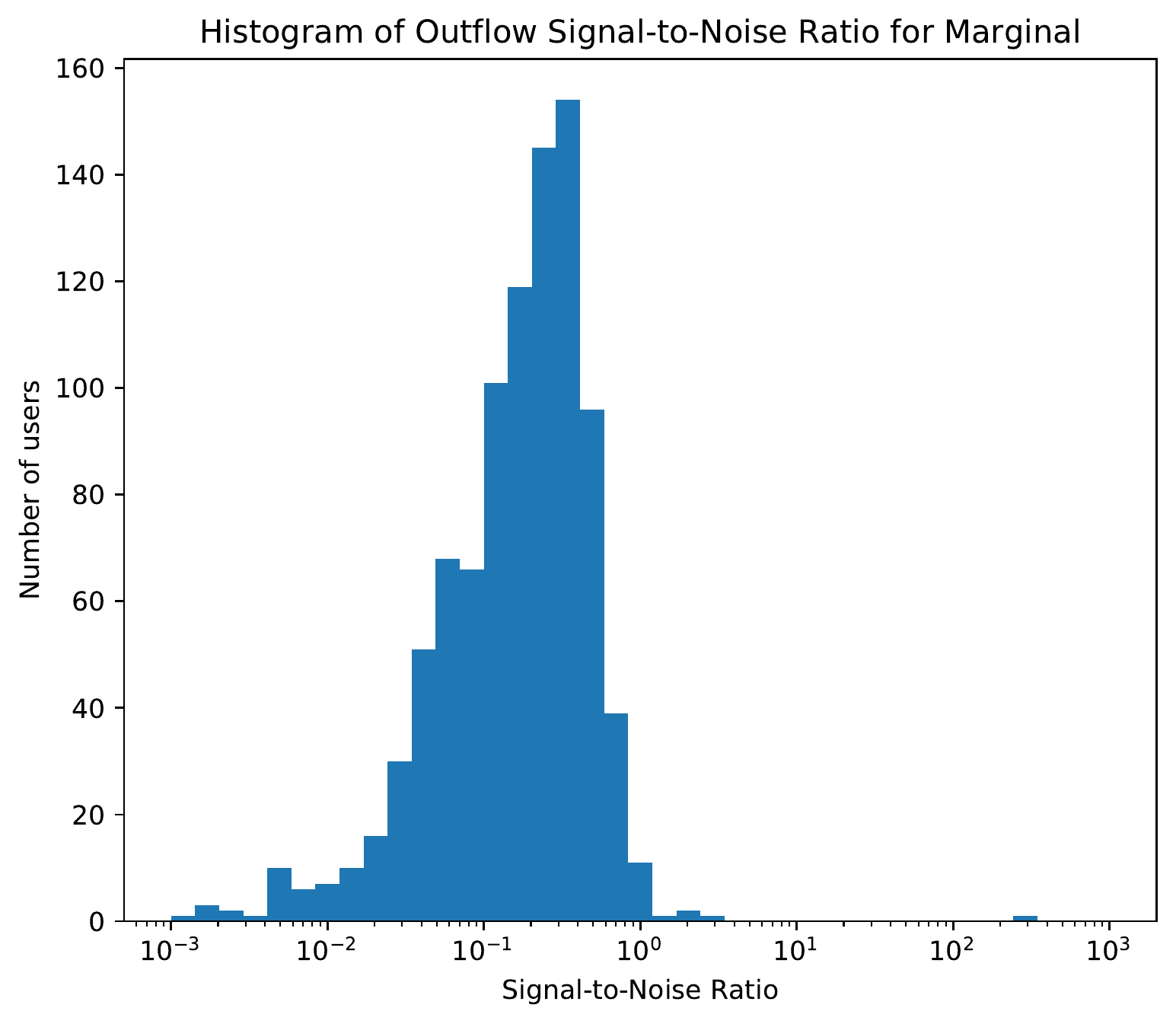}}\quad %
    \subfloat[\centering \trackin{}]{\includegraphics[width=0.45\textwidth]{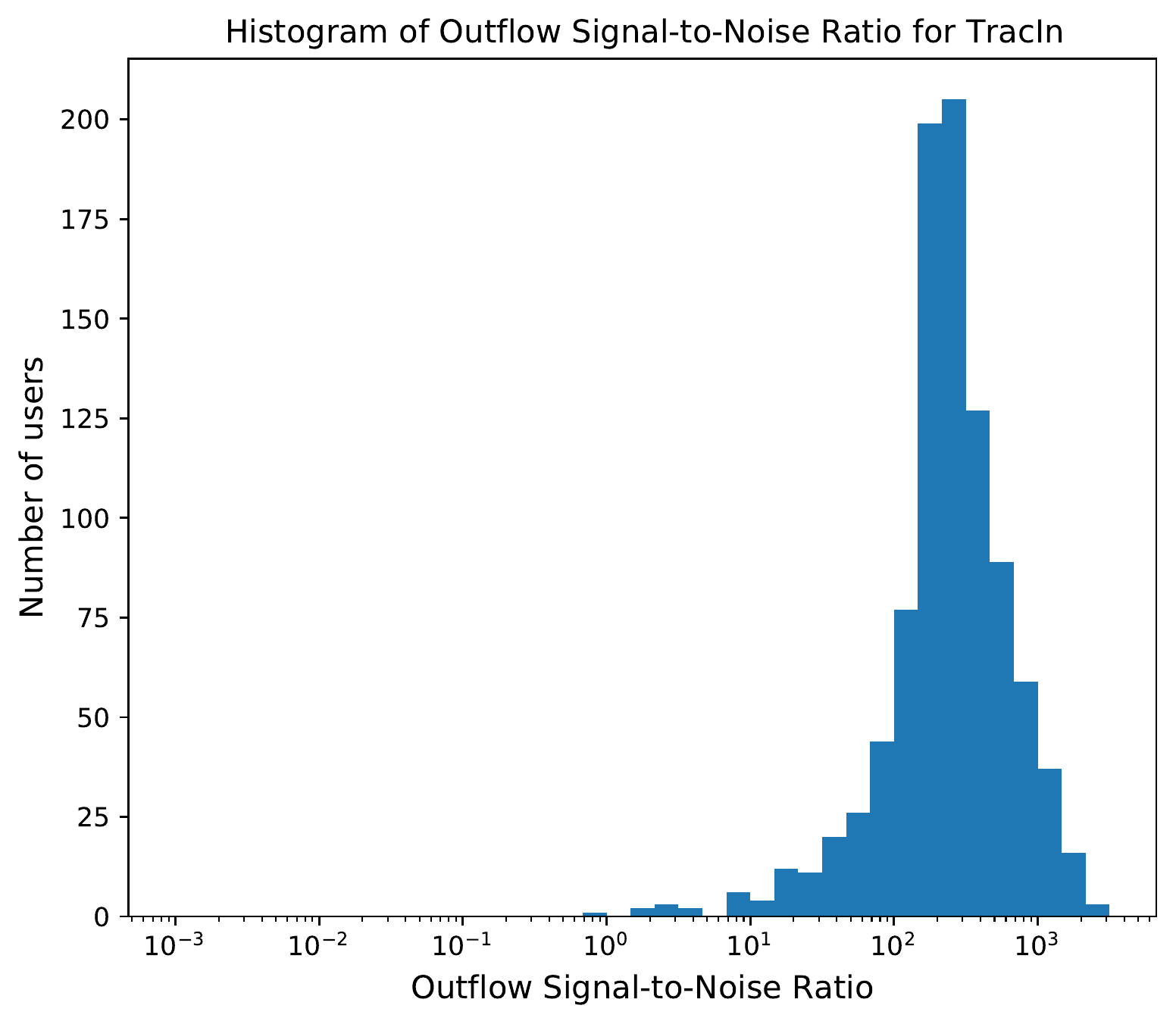}}
    \caption{Histogram of user outflow signal-to-noise ratios, computed across ten splits.}
    \label{fig:noise}
\end{figure}

Figure~\ref{fig:noise} shows the distribution of the signal-to-noise ratio of outflows across individuals, both when using \marginal{} and \trackin{} as measures of influence. For \marginal, almost all individuals have a ratio less than one, and 90\% of individuals have a ratio less than $0.49$. For \trackin, 90\% of individuals  have a ratio greater than $76$. A similar observation holds for inflows.

Due to its low signal-to-noise ratio, we do not present reciprocity results using \marginal.

\subsubsection{Reciprocity using \trackin{} Influence}

We start with some sanity checks. We have already seen that \trackin{} has a high signal-to-noise ratio across splits. Another potential source of measurement error is the first-order approximation of \trackin{} (recall Remark~\ref{rem:first-order}). We measure this discrepancy, see Figure~\ref{fig:approximation}. We find that the percentage relative discrepancy remains relatively small; its 80th percentile (across gradient descent steps) is 1.1\%.

\begin{figure}
\centering
\subfloat[\centering Single split]{\includegraphics[width=0.48\textwidth]{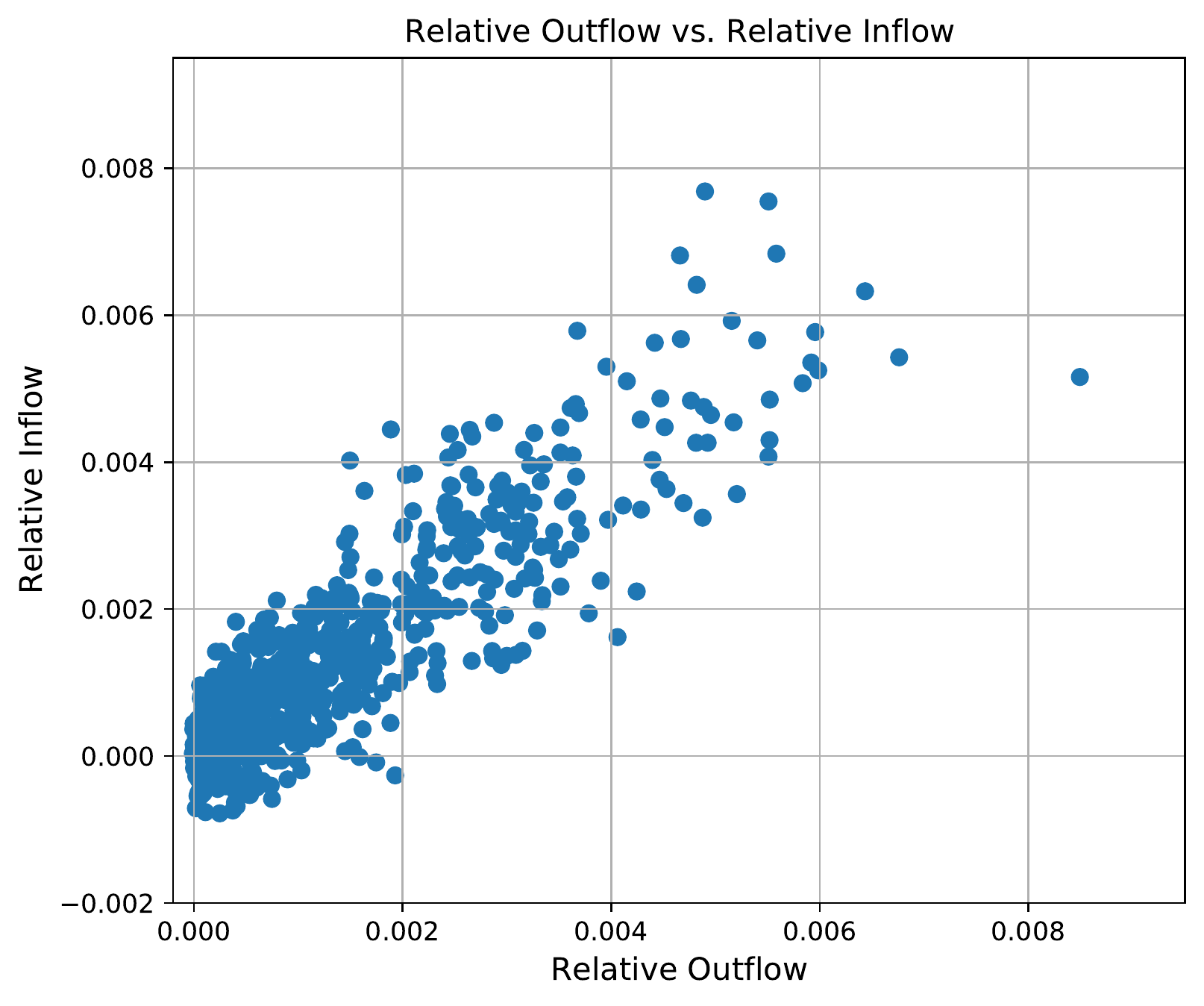}}\quad %
\subfloat[\centering Ten splits]{\includegraphics[width=0.48\textwidth]{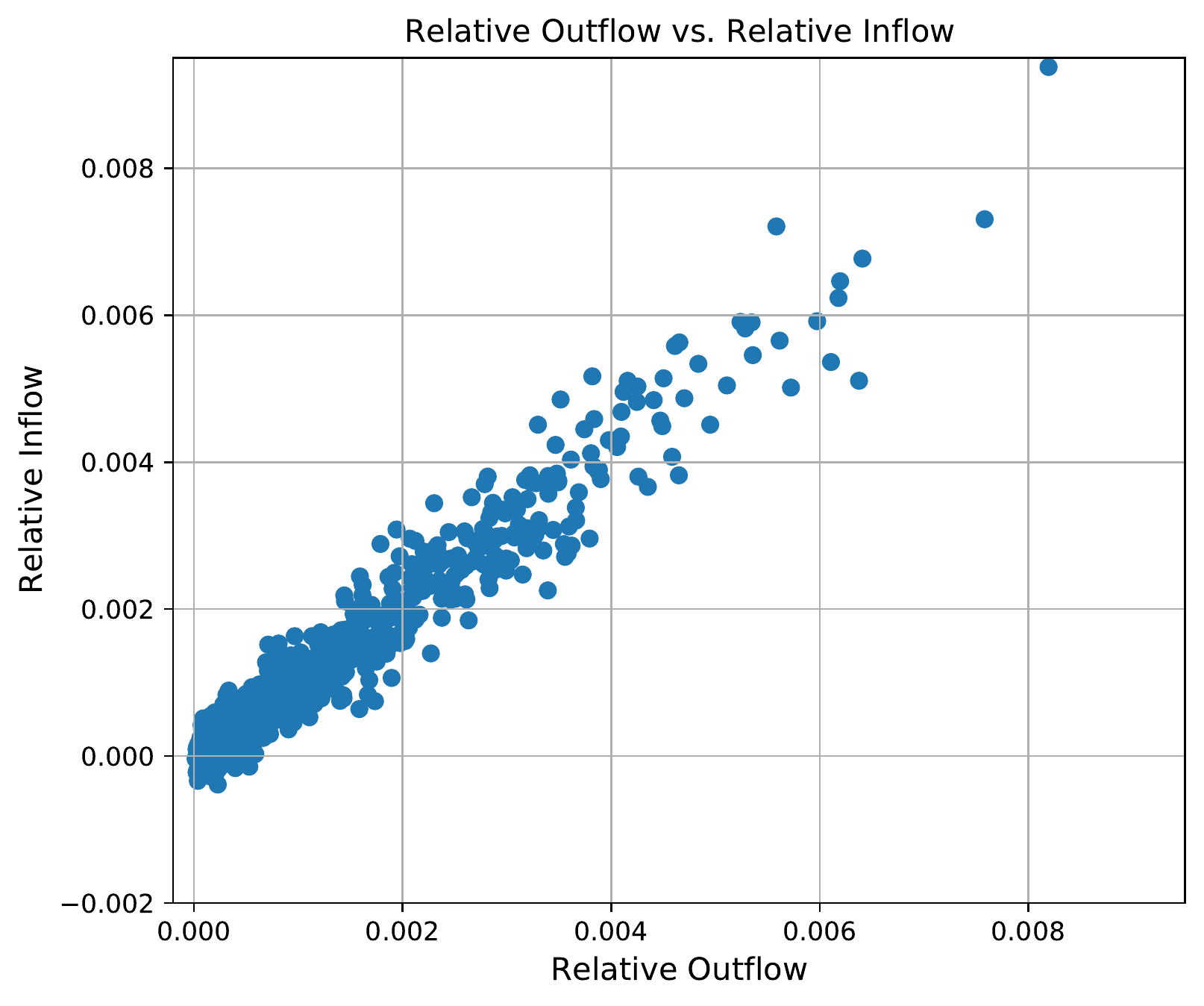}}
\caption{Average inflow vs. outflow for all individuals. The average is over ten training runs for a single split (left), and over all runs for all splits (right).}
\label{fig:inflow_outflow}
\end{figure}

Next, we compute average outflows and inflows and inspect them across all individuals. We run the experiment on 10 different splits, and on each split we average the measurements across 10 random initializations. See Figure~\ref{fig:inflow_outflow}. We notice that inflow is largely commensurate to outflow. 

\begin{figure}
\centering
\subfloat[\centering Single split]{\includegraphics[width=0.48\textwidth]{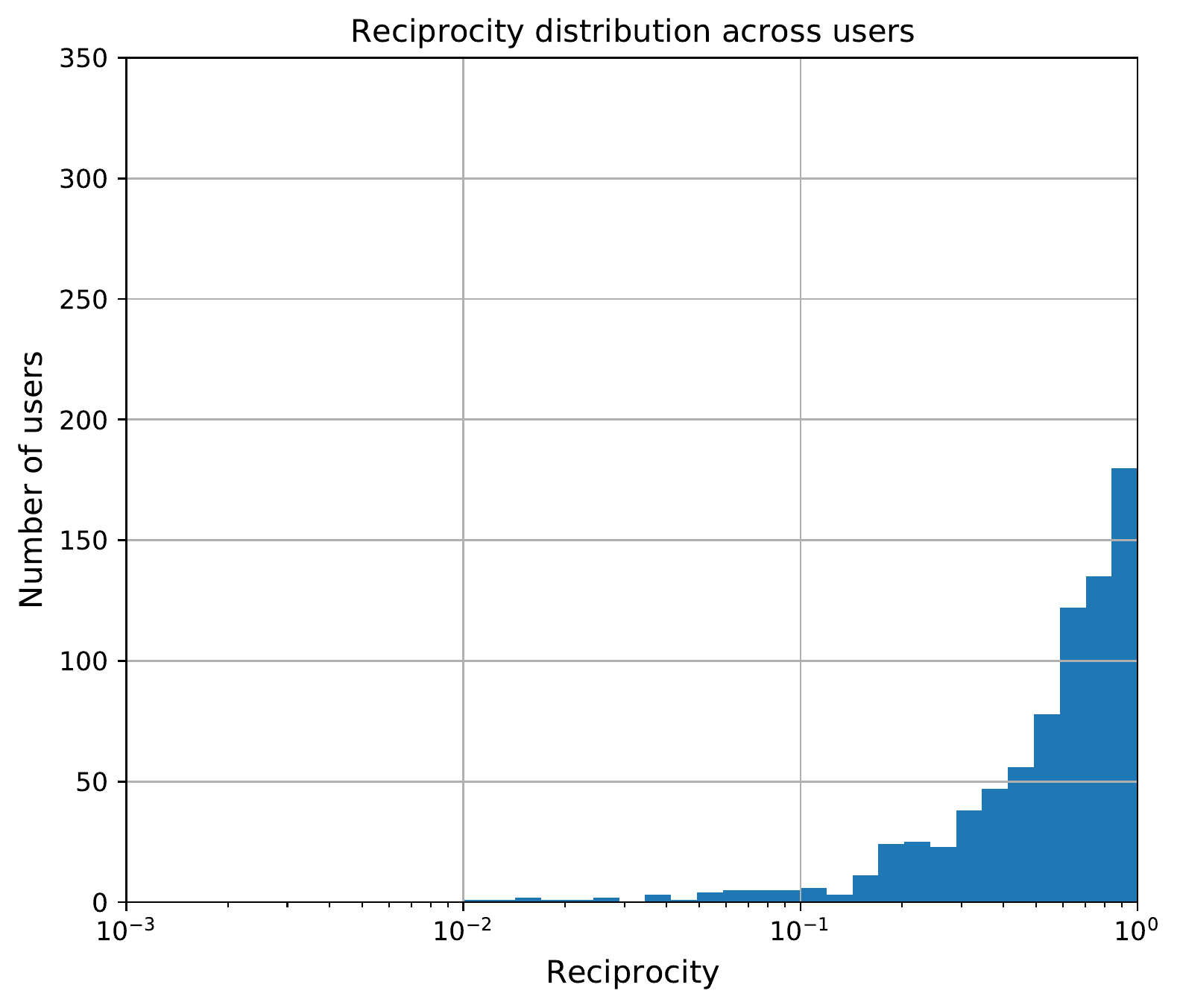}}\quad %
\subfloat[\centering Ten splits]{\includegraphics[width=0.48\textwidth]{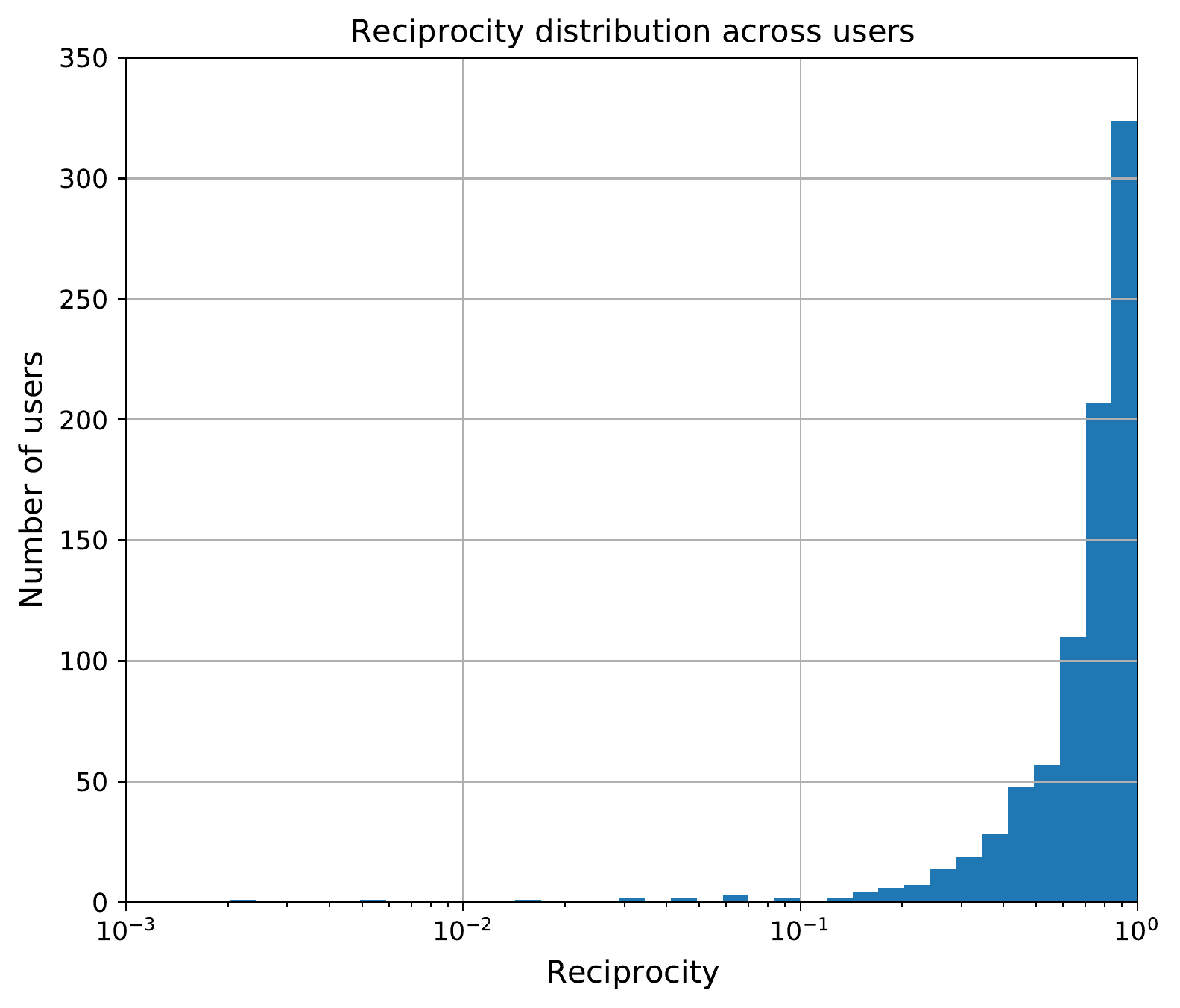}}
\caption{Histogram of Reciprocity across individuals, computed on a single split (left) and across ten splits (right).}
\label{fig:hist}
\end{figure}


For a single split, we find that 75\% of the individuals have reciprocities in $[0.2, 1]$, i.e. the model is $(0.75, 0.2)$-\trackin-reciprocal. We find that the correlation between inflow and outflow is $0.89$.

When averaging across ten splits, 75\% of the individuals have reciprocities in $[0.48, 1]$, i.e. the model is $(0.75, 0.48)$-\trackin-reciprocal. We find that the correlation between inflow and outflow is $0.97$.

Averaging across splits can be viewed as approximating the expectation of reciprocity, and this is in line with the theoretical results. The fact that we measure some reciprocity even on a single split suggests that the result holds even on one realization, and not just in expectation, possibly because each individual contributes many data points in this data set.

We also observe that inflow and outflow are largely non-negative. In a single split, 0.5\% of individuals have a negative inflow and 16.9\% of individuals have a negative outflow. When averaging across all splits, all individuals have a positive inflow, and 11.1\% of individuals have a negative outflow; on average, the presence of these individuals degrades the prediction quality of other individuals. However, these individuals have small magnitudes of outflow and inflow; this also explains the difference between the correlation measure (which is dominated by individuals with large inflow and outflow) and the $(p,\alpha)$ measure.

\subsection{Healthcare Data Sets}
\label{sec:health}

We investigate reciprocity in two healthcare machine learning data sets. The first is a data set from~\cite{diabetes} about predicting diabetes. It has ten features: age, sex, body mass index, average blood pressure, and six blood serum measurements, and the task is to predict disease progression one year after the time of the readings. The data is from 442 individuals. The second is a data set from the UCI Machine Learning Repository about predicting breast-cancer. There are thirty features that relate to geometric properties of cell nuclei from a digitized image of a fine needle aspirate (FNA) of a breast mass. The task is to predict whether the breast cancer is present or not (malignant or benign). The data is from 569 individuals.

In both data sets, each individual corresponds to a single data point, unlike the Movielens data set where individuals correspond to at least 20 data points. Consequently, every individual belongs to exactly one of the training or deployment sets; individuals in the training set only have outflows and individuals in the deployment set only have inflows. In this case, measuring reciprocity is only possible in expectation over random splits into the training or deployment sets. We mimic this by averaging measurements over 100 random splits of the data (see Remark~\ref{rem:dist}).


On the diabetes prediction task, we train a linear regression model optimized for the mean squared error, with a number of steps $T = 200$, and a learning rate $\eta = 0.01$. On the breast cancer classification task, we train a logistic regression model with a number of steps $T = 600$ and a learning rate $\eta = 0.1$. In both cases, some features have very different scales, so we found it important to normalize all features (using mean and variance computed on the training set).

To have a reasonable baseline at initialization, we initialize the model parameters in such a way that it predicts the average training label (by setting the bias term to a constant and other parameters to zero).

\subsubsection{Reciprocity using \marginal{} Influence}
First, we measure \marginal{} reciprocity. Recall from Remark~\ref{rem:marginal-noise} that this method is susceptible to noise. We find that the results are not as noisy as in the MovieLens experiment, possibly because these data sets are smaller; the 90-th percentile of signal-to-noise ratio is 0.55 for the diabetes model and 0.58 for the breast cancer model, compared to 0.49 for MovieLens.

Despite the noise, we report the results so that we can compare to \trackin{}. See Figure~\ref{fig:health_marginal}. We find that the diabetes model is $(0.75, 0.36)$-\marginal-reciprocal, and the breast cancer model is $(0.75, 0.48)$-\marginal-reciprocal.

\begin{figure}
\centering
\subfloat[Diabetes]{\includegraphics[width=0.48\textwidth]{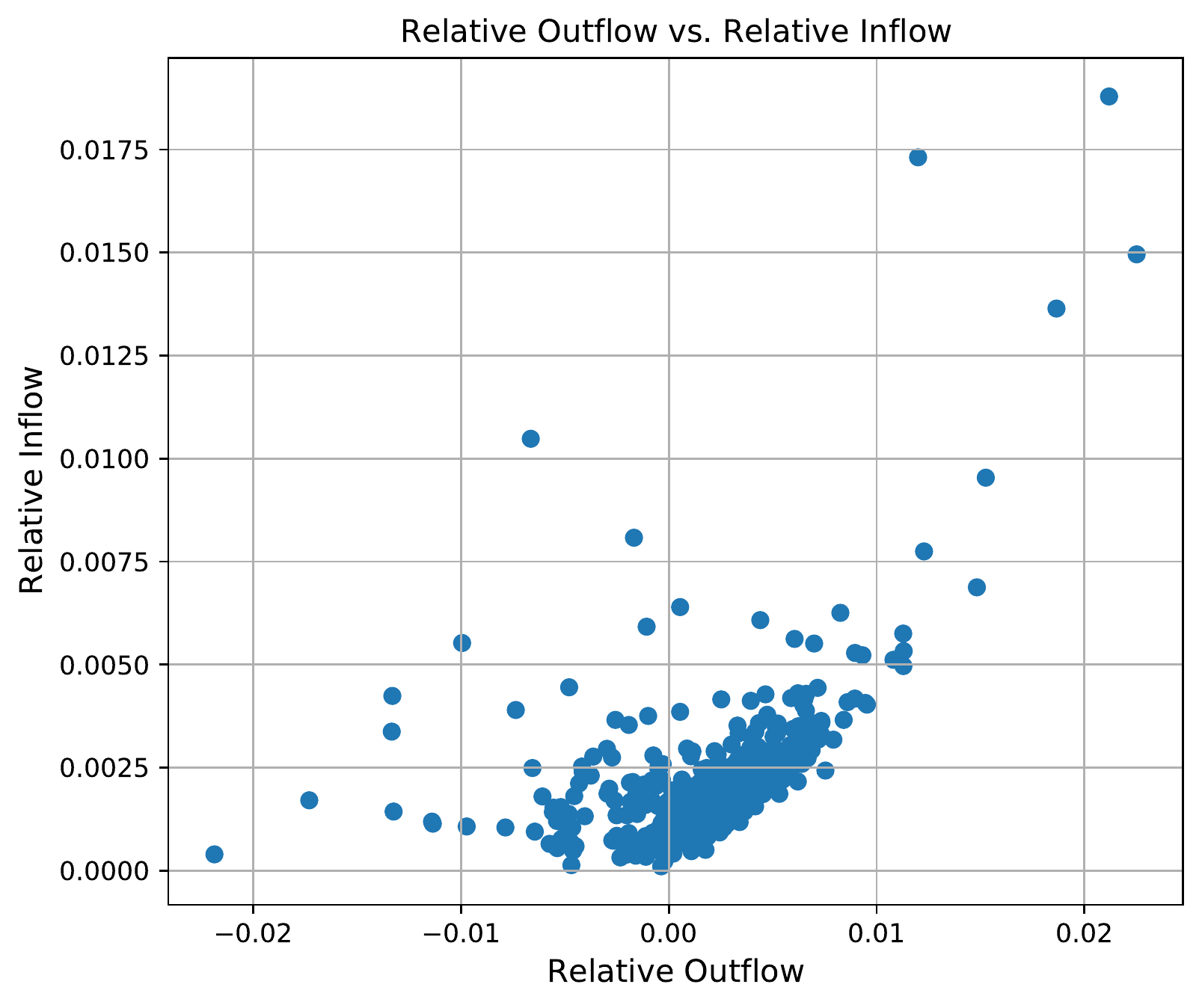}}\quad %
\subfloat[Breast cancer]{\includegraphics[width=0.48\textwidth]{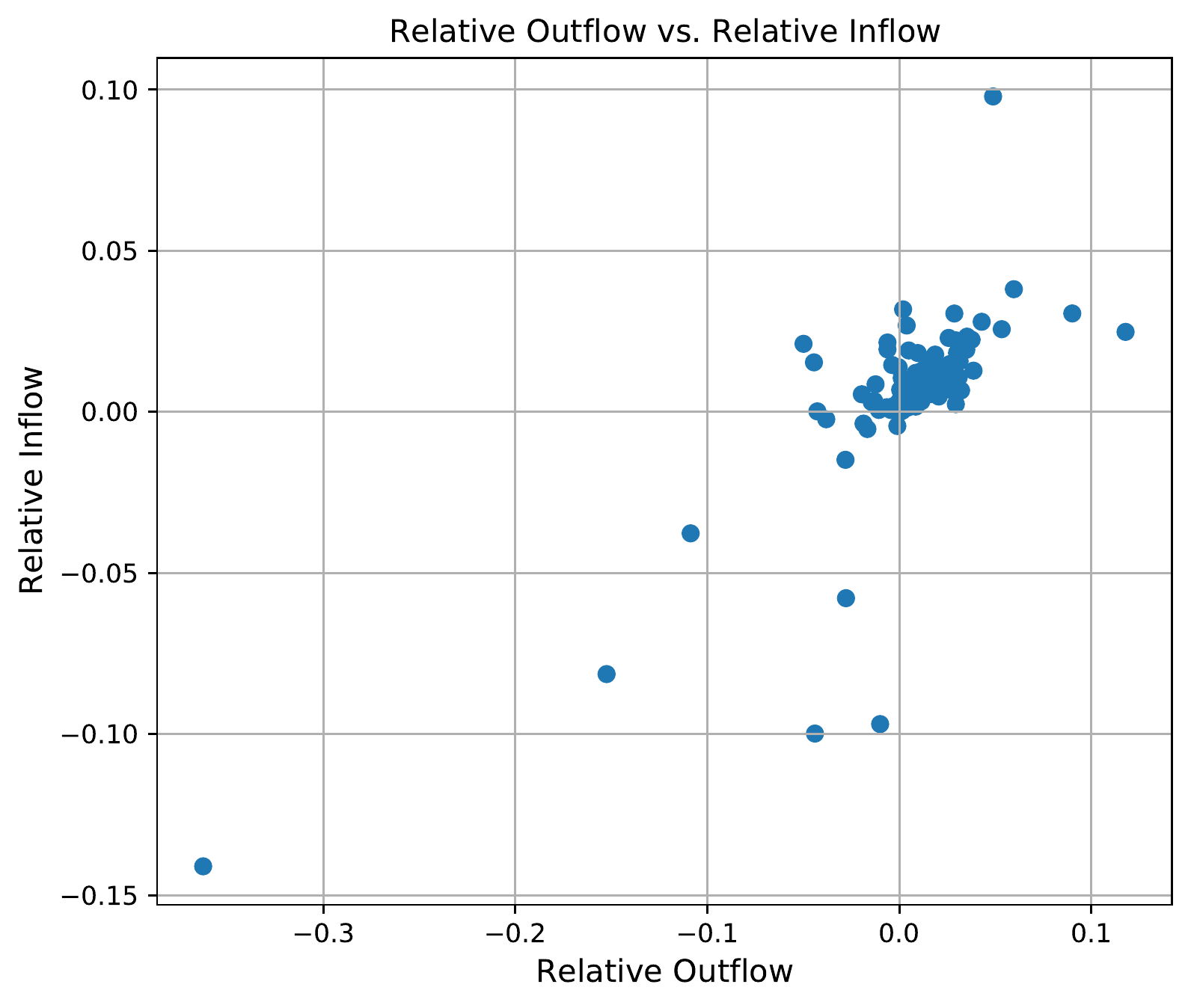}}
\caption{Average \marginal-inflow vs. \marginal-outflow for all individuals. }
\label{fig:health_marginal}
\end{figure}

\subsubsection{Reciprocity using \trackin{} Influence}

Next, we compute outflows and inflows using \trackin{} as a measure of influence. See Figure~\ref{fig:reciprocity_scatter}.

The diabetes model is $(0.75, 0.76)$-\trackin{}-reciprocal, and the breast cancer model is $(0.75, 0.93)$-\trackin{}-reciprocal. Both are higher than \marginal-reciprocities.

We notice that for the second model, one individual has a very large negative inflow and outflow. This data point appears to be particularly hard to classify, as it has the largest average deployment loss (more than twice as large as the next largest loss). This large loss induces a large gradient which translates to a large negative inflow and outflow.

We also observe that, as the training proceeds, inflows and outflows become larger in magnitude, even though the model parameters may not substantially change. Figure~\ref{fig:reciprocity_scatter_2} shows one such example. When the diabetes model is trained for 1000 steps, there is a slight increase in deployment loss (about 1\%), but a more significant effect on reciprocity: the 75th percentile of reciprocity decreases from $0.76$ to $0.20$. This drop may be due to large credits (inflows and outflows) being assigned because individual gradients are large, even though in aggregate, the gradient is small, hence changes in model parameters are small.

\begin{figure}
\centering
\subfloat[Diabetes]{\includegraphics[width=0.48\textwidth]{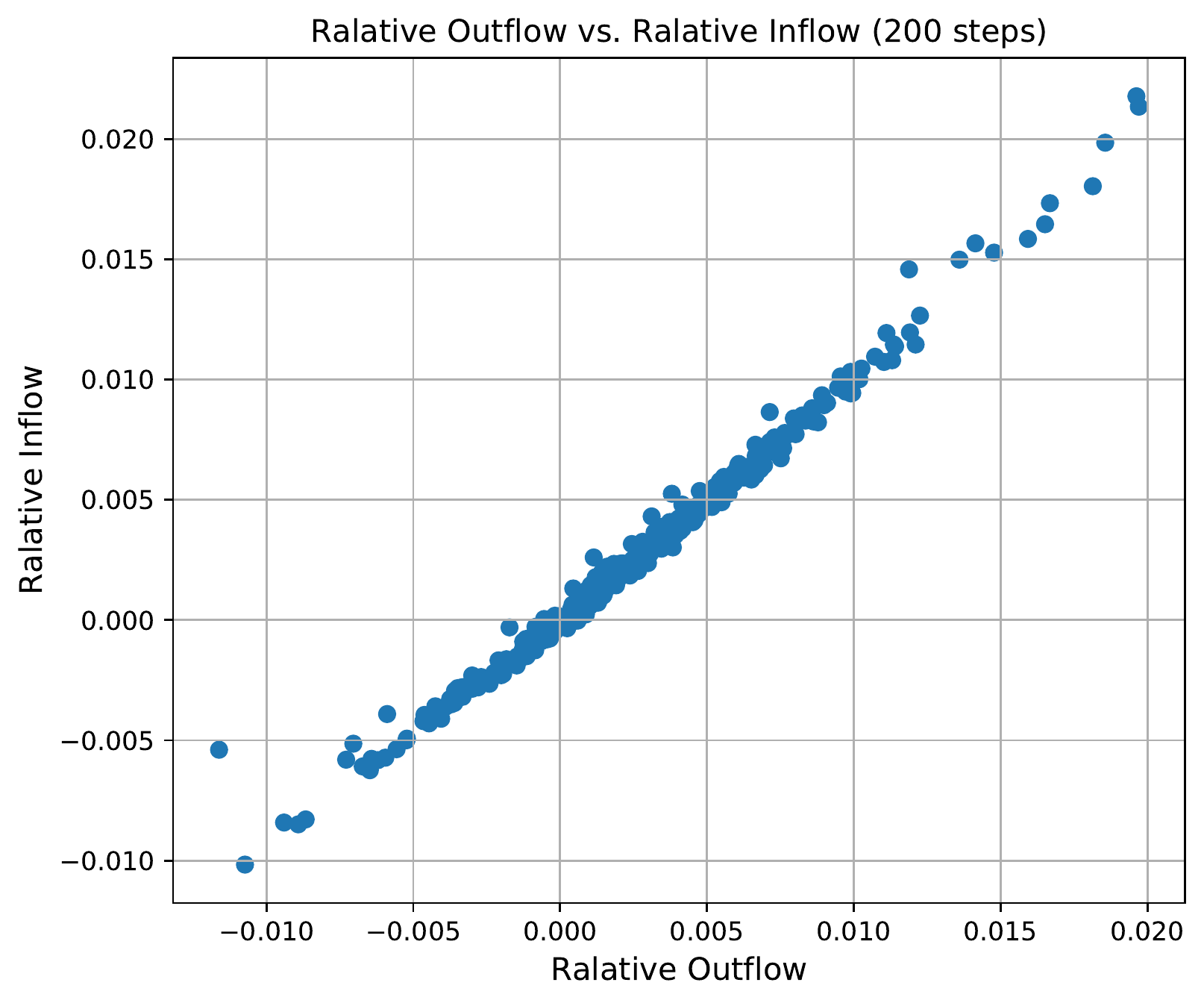}}\quad %
\subfloat[Breast cancer]{\includegraphics[width=0.48\textwidth]{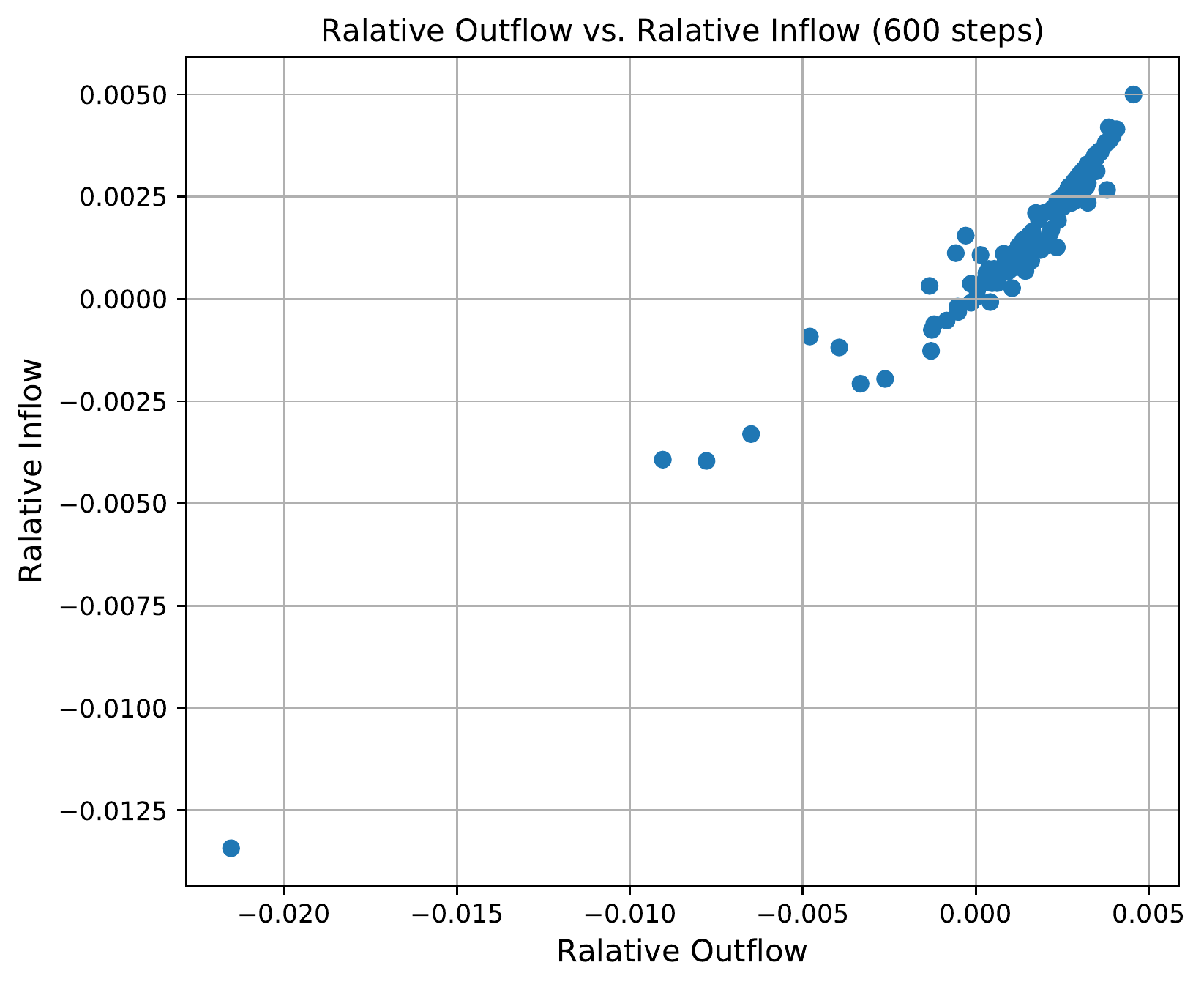}}
\caption{Average inflow vs. outflow for the diabetes model (left) and the breast cancer model (right).}
\label{fig:reciprocity_scatter}
\end{figure}

\begin{figure}
\centering
\subfloat[Diabetes]{\includegraphics[width=0.48\textwidth]{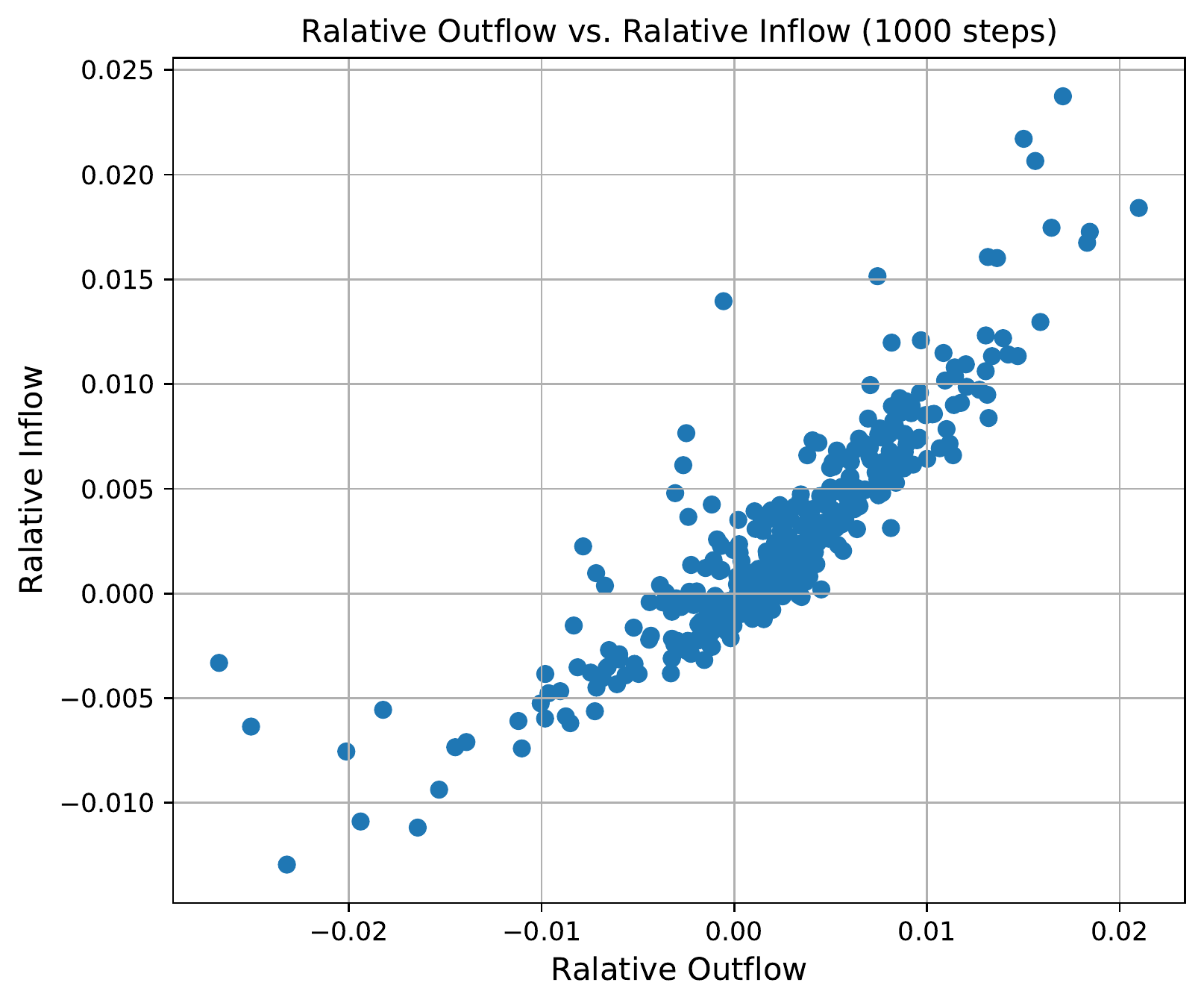}}\quad %
\subfloat[Breast cancer]{\includegraphics[width=0.48\textwidth]{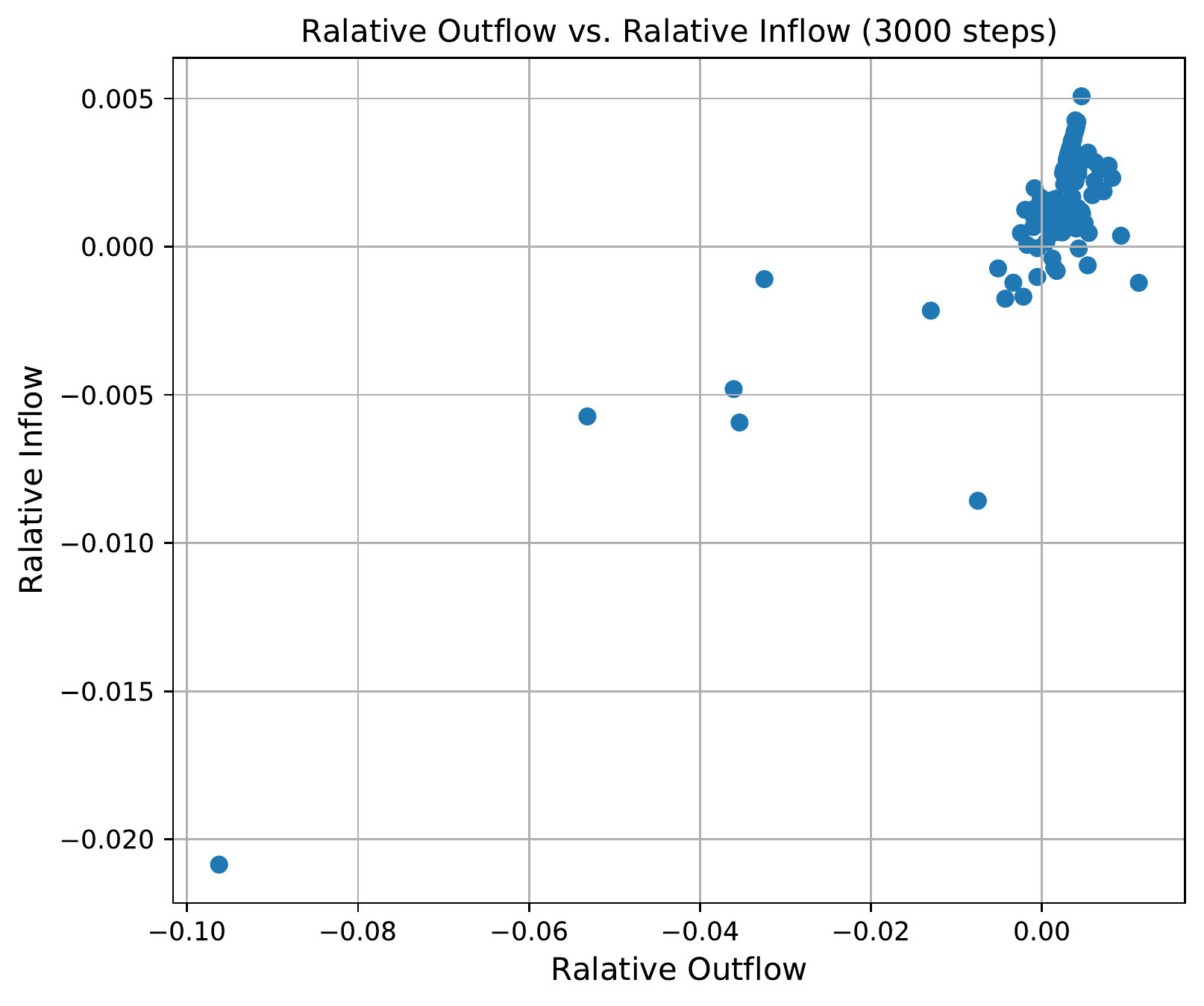}}
\caption{When the models are trained for much longer after convergence, reciprocity is gradually lost.}
\label{fig:reciprocity_scatter_2}
\end{figure}

\section{Conclusion and Open Directions}
\label{sec:open}

Individuals often contribute data to machine-learning models while also benefiting from the predictions of these models. We use measures of training data influence to propose measures of the contribution (outflow of influence) and the benefit (inflow of influence), and the balance between outflow and inflow. We found that models trained using variants of Gradient Descent---under assumptions that the deployment distribution is similar to the training distribution, formalized in Theorem~\ref{thm:main}---are highly reciprocal. There are several open directions. 

Our measure and claims around reciprocity are based on \emph{how} we compute influence. Are there other measures of influence besides \marginal{} and \trackin? 
One candidate is the Shapley value. The challenge however is in making it computationally feasible; every evaluation of the set function (over training examples belonging to a set of individuals) involves a costly training run. Perhaps there is a clever way to make this cheaper in a restricted class of models.


When the distributional assumption is not satisfied, can we modify training algorithms to enforce reciprocity? For instance, by sampling individual's data or by modifying gradients, similarly to techniques used in differential privacy~\cite{dpsgd}
.


\bibliographystyle{ACM-Reference-Format}
\bibliography{references}

\appendix

\section{Additional Experiments}

\paragraph{Sanity Checks on MovieLens:} We measure the noise in the inflows and outflows due to randomness in the training process.

\begin{figure}
    \centering
    \subfloat[\centering outflows]{{\includegraphics[width=5cm]{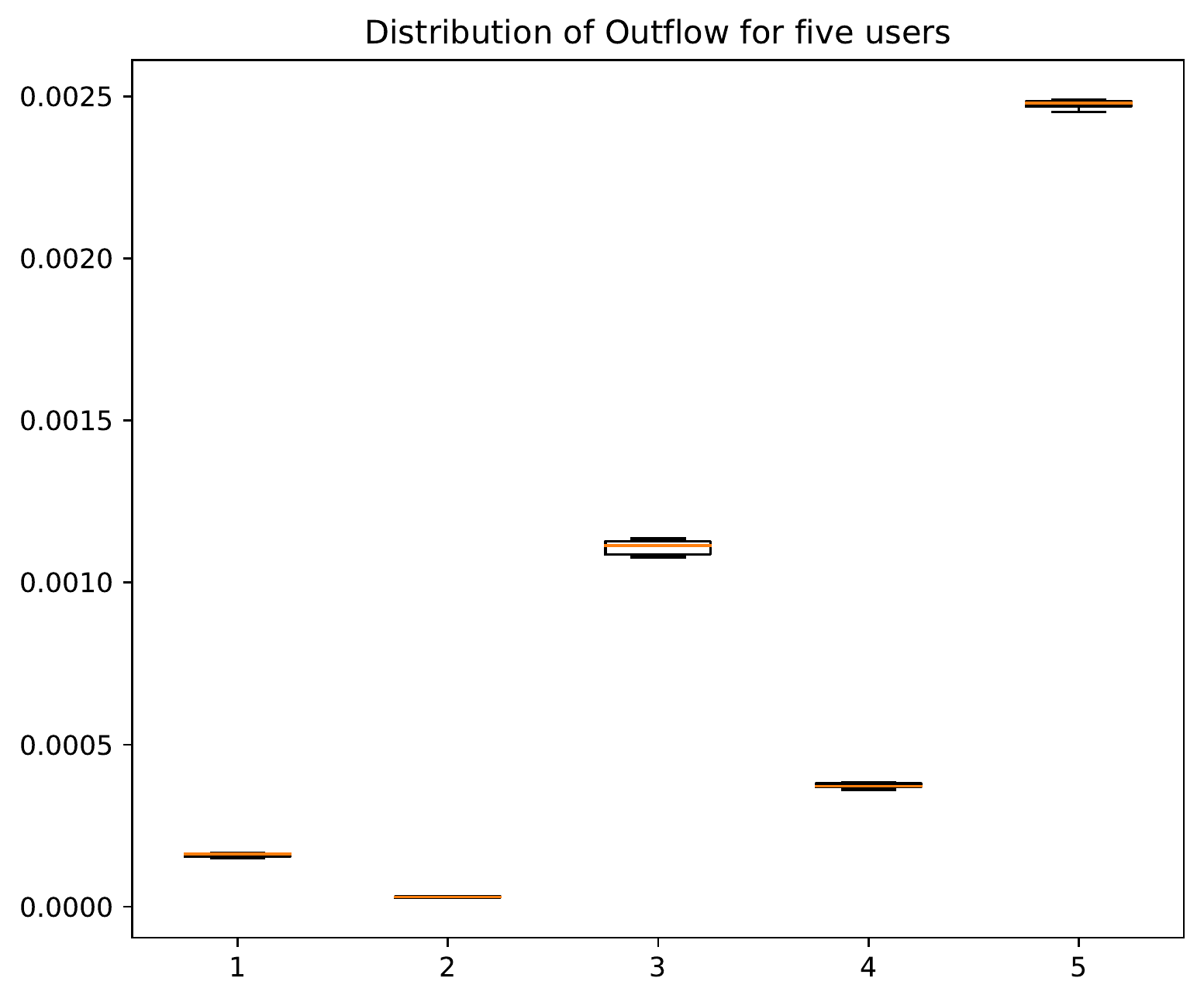}}}%
    \qquad
    \subfloat[\centering inflows ]{{\includegraphics[width= 5cm]{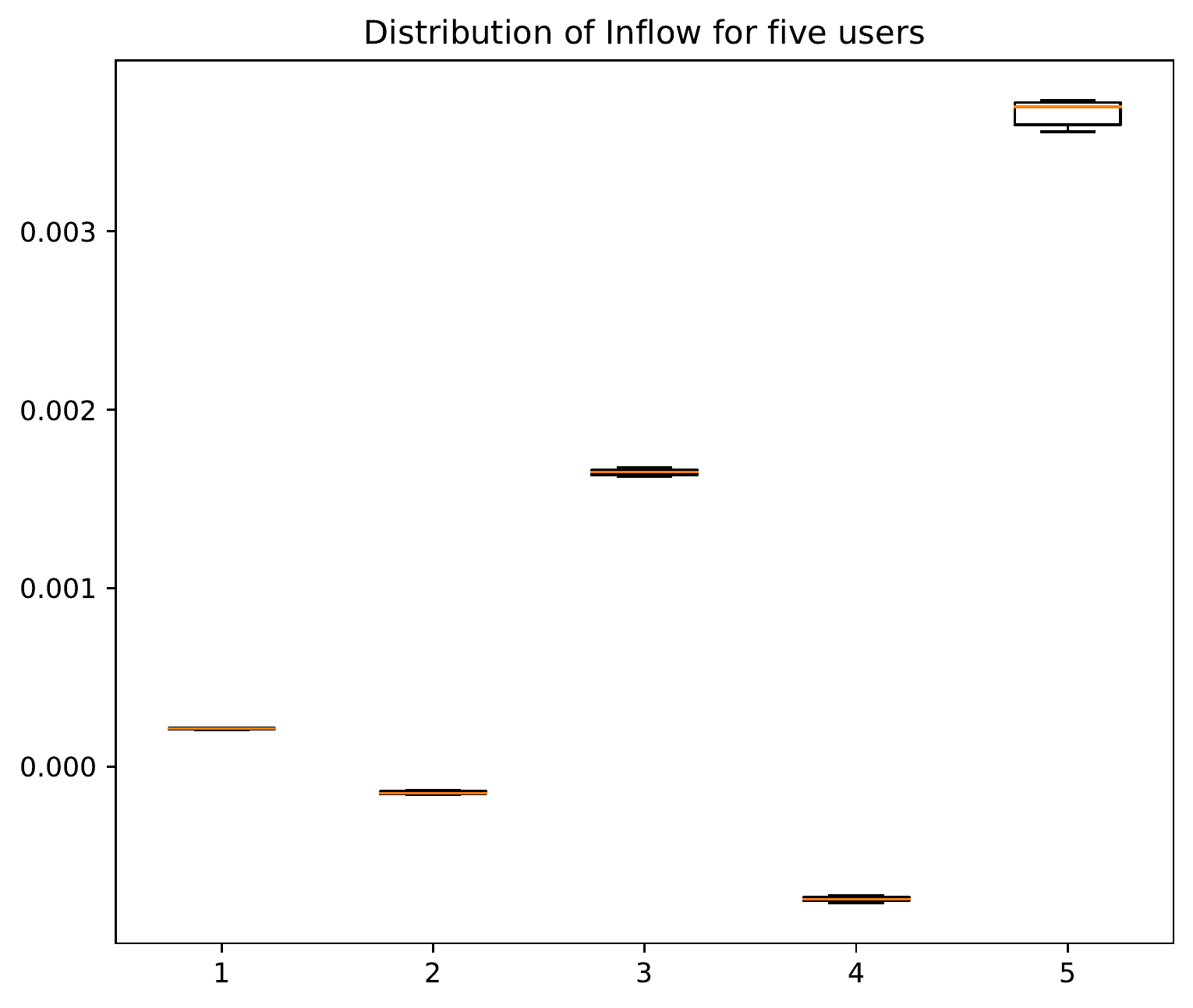}}}%
    \caption{Distribution of inflow and outflows across 10 runs for 5 individuals. The box edges show the upper and lower quartiles, the whiskers show the fifth and ninety-fifth percentile.
    }%
    \label{fig:five-examples}%
\end{figure}

We plot the distributions across five arbitrarily chosen individuals. See Figure~\ref{fig:five-examples}. To aid interpretation, we normalize the inflows and outflows by the total inflow (which approximates the total change in loss). Therefore, one can read the numbers as the fraction of overall inflow, (or approximately the fraction of total loss reduction). We observe that the inflows and outflows are consistent across runs, however, there is some variation (stemming from the random initialization of the embeddings).

\begin{figure}
    \centering
    \includegraphics[width=0.45\textwidth]{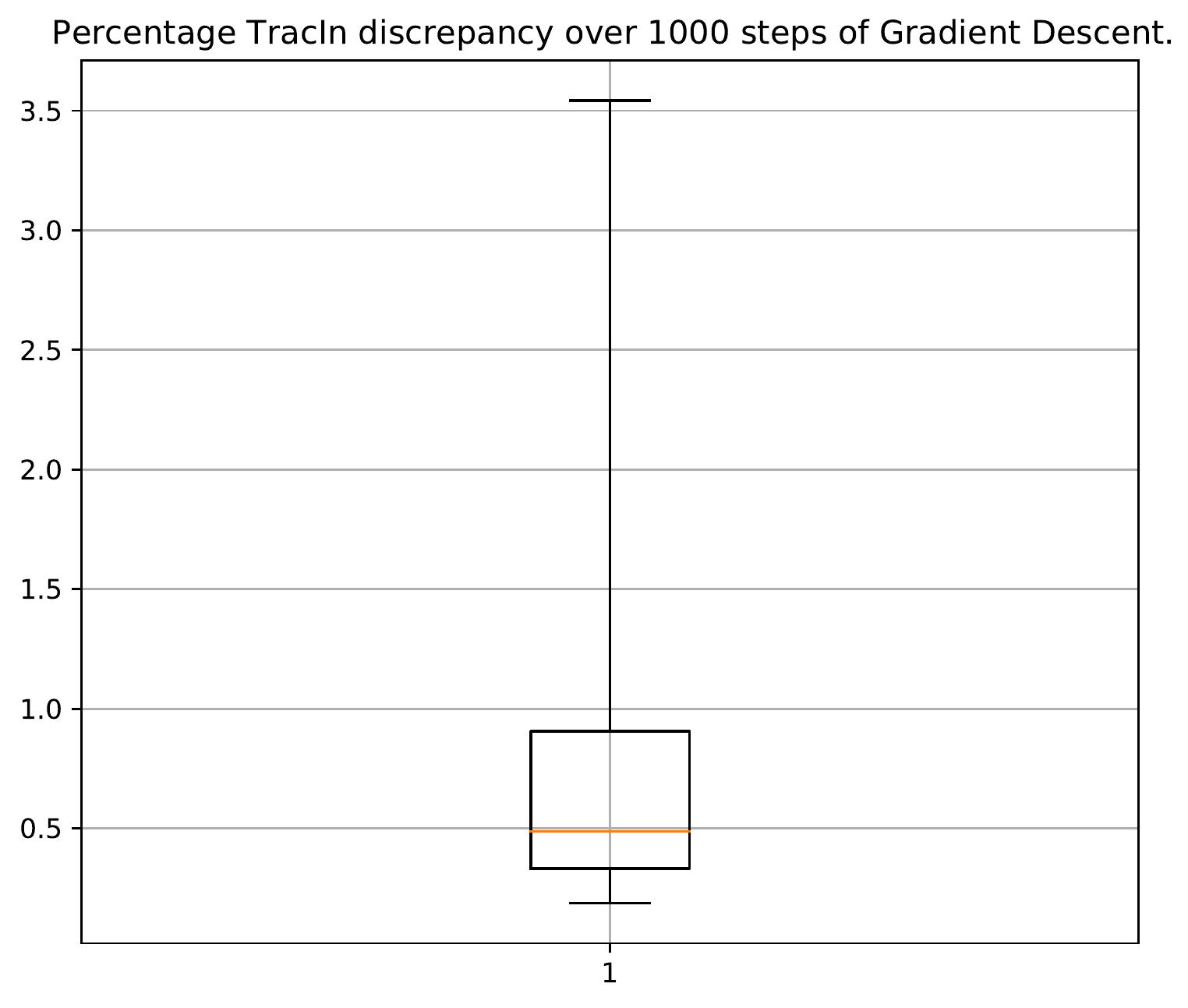}
    \caption{\trackin{} approximation discrepancy across 1000 steps of Gradient Descent.}
    \label{fig:approximation}
\end{figure}

Next, we study the discrepancy in the measurement of outflows due to the first-order approximation of \trackin{} (recall Remark~\ref{rem:first-order}). Figure~\ref{fig:approximation} shows the relative discrepancy for one run. The numerator is the sum of the gradient dot products (in Equation~\ref{eq:tracin}) across examples for one step of gradient descent minus the total change in loss across $Z'$. The denominator is the total change in loss across $Z'$. The percentage relative discrepancy remains relatively small; its 80th percentile is 1.1\%. This can be further reduced by using a smaller learning rate.

\end{document}